%% file: main.tex
\newif\ifnips

\nipsfalse % for full version

\ifnips
\documentclass{article}

\usepackage{neurips_2020}

\usepackage[utf8]{inputenc} % allow utf-8 input
\usepackage[T1]{fontenc}    % use 8-bit T1 fonts
\usepackage{hyperref}       % hyperlinks
\usepackage{url}            % simple URL typesetting
\usepackage{booktabs}       % professional-quality tables
\usepackage{amsfonts}       % blackboard math symbols
\usepackage{nicefrac}       % compact symbols for 1/2, etc.
\usepackage{microtype}      % microtypography

\else

\documentclass{article}

\usepackage[a4paper,top=3.5cm,bottom=3cm,left=3.5cm,right=3.5cm,marginparwidth=1.5cm]{geometry}
\usepackage{hyperref}       % hyperlinks
\usepackage{url}            % simple URL 

\fi

\usepackage{amsthm}
\usepackage{bbm}
\usepackage{times}
\usepackage{graphicx,color}
\usepackage{array,float}
\usepackage{xspace}
\usepackage{xcolor}
\usepackage{amstext,amssymb,amsmath}
\usepackage{hyphenat}
\usepackage{verbatim}
\usepackage{bm}
\usepackage{paralist}
\usepackage{ulem}
\usepackage{todonotes}
\usepackage{paralist}\usepackage{bbm}
\usepackage{wrapfig}

\usepackage[noend]{algorithmic}
\usepackage{algorithm}
\usepackage{xparse,etoolbox}
\usepackage{threeparttable}
\usepackage{bbm}
\usepackage{enumitem}
\usepackage{dsfont}

\newtheorem{lem}{Lemma}[section]

\newtheorem{thm}[lem]{Theorem}

\newtheorem{cor}[lem]{Corollary}

\newtheorem{defn}[lem]{Definition}

\newtheorem{obv}[lem]{Observation}

\newtheorem{claim}[lem]{Claim}

\renewcommand{\paragraph}[1]{\vspace{3pt}\noindent\textbf{#1}}

\newcommand{\cO}{\ensuremath{\mathcal{O}}}
\newcommand{\cX}{\ensuremath{\mathcal{X}}}
\newcommand{\cY}{\ensuremath{\mathcal{Y}}}
\newcommand{\cP}{\ensuremath{\mathcal{P}}}

\newcommand{\cZ}{\ensuremath{\mathcal{Z}}}		

\newcommand{\hS}{\widehat S}

\newcommand{\wtc}{\widetilde{\cC}}

\newcommand{\rel}{\mathsf{LearnHalf}}

\newcommand{\npub}{n_\mathsf{pub}}

\newcommand{\nprv}{n_\mathsf{priv}}

\newcommand{\prv}{\mathsf{priv}}

\newcommand{\erm}{\mathsf{ERM}}

\newcommand{\hrf}{h_{S'}^\erm}

\DeclareMathOperator*{\argmin}{arg\,min}

\newcommand{\pr}[2]{\underset{#1}{\mathbb{P}}\left[ #2 \right]}
\newcommand{\ex}[2]{\underset{#1}{\mathbb{E}}\left[ #2 \right]}

\newcommand{\eps}{\epsilon}

\newcommand{\cA}{\mathcal{A}}

\newcommand{\cD}{\mathcal{D}}

\newcommand{\cG}{\mathcal{G}}

\newcommand{\bw}{\mathbf{w}}

\newcommand{\cT}{\mathcal{T}}

\newcommand{\ind}{{\mathbf{1}}}

\newcommand{\cV}{\mathcal{V}}

\newcommand{\cF}{\mathcal{F}}

\newcommand{\conC}{\mathsf{ConstrHalf}}

\newcommand{\as}{\mathsf{Aff}}

\newcommand{\re}{\mathbb{R}}

\newcommand{\cB}{\mathcal{B}}

\newcommand{\ignore}[1]{}

\newcommand{\cW}{\mathcal{W}}

\newcommand{\cR}{\mathcal{R}}

\newcommand{\cC}{\mathcal{C}}
\newcommand{\whh}{\widehat{g}}

\newcommand{\Sprv}{S_{\sf priv}}
\newcommand{\Spub}{S_{\sf pub}}
\newcommand{\uSpub}{\tilde{S}_{\sf pub}}
\newcommand{\pub}{{\sf pub}}

\newcommand{\Dprv}{\cD_{\sf priv}}
\newcommand{\Dpub}{\cD_{\sf pub}}

\newcommand{\tcD}{\tilde{\cD}}

\newcommand{\err}{\mathsf{err}}
\newcommand{\herr}{\widehat{\mathsf{err}}}

\newif\ifnotes
\notestrue
\ifnotes
\usepackage{todonotes}

\newcommand{\rnote}[1]{ [\textcolor{magenta}{Raef: #1}] }

\newcommand{\snote}[1]{ [\textcolor{blue}{Shay: #1}] }

\else

\newcommand{\rnote}[1]{}
\newcommand{\nnote}[1]{}
\newcommand{\snote}[1]{}
\fi
    
\newcommand{\stabname}{determined\xspace}

\title{Learning from Mixtures of Private and Public Populations}

\ifnips

\author{%
	Raef Bassily\thanks{Department of Computer Science \& Engineering, The Ohio State University. \texttt{bassily.1@osu.edu}} 
		\and   Shay Moran\thanks{Department of Computer Science, Princeton University.  \texttt{shaymoran1@gmail.com}}
	\and Anupama Nandi \thanks{Department of Computer Science \& Engineering, The Ohio State University. \texttt{nandi.10@osu.edu}}
}

\else
\author{%
	Raef Bassily\thanks{Department of Computer Science \& Engineering, The Ohio State University. \texttt{bassily.1@osu.edu}} 
		\and   Shay Moran\thanks{Department of Computer Science, Princeton University.  \texttt{shaymoran1@gmail.com}}
	\and Anupama Nandi \thanks{Department of Computer Science \& Engineering, The Ohio State University. \texttt{nandi.10@osu.edu}}
}
\date{}
\fi

\begin{document}

\maketitle

\begin{abstract}
We initiate the study of a new model of supervised learning under privacy constraints.
    Imagine a medical study where a dataset is sampled from a population of both healthy and unhealthy individuals. 
    Suppose healthy individuals have no privacy concerns (in such case, we call their data ``public'') 
    while the unhealthy individuals desire stringent privacy protection for their data. 
    In this example, the population (data distribution) is a mixture of private (unhealthy) and public (healthy) sub-populations that could be very different. 

Inspired by the above example, we consider a model in which
%   the population $\cD$ is a \emph{mixture} of two possibly \emph{distinct} sub-populations: 
    the population $\cD$ is a \emph{mixture} of two sub-populations: 
    a private sub-population $\Dprv$ of private and sensitive data, 
    and a public sub-population $\Dpub$ of data with no privacy concerns. 
    Each example drawn from $\cD$ is assumed to contain a privacy-status bit 
    that indicates whether the example is private or public. 
    The goal is to design a learning algorithm that satisfies differential privacy 
    only with respect to the private examples.

Prior works in this context assumed a homogeneous population where private and public data arise from the same distribution,
        and in particular designed solutions which exploit this assumption.
        We demonstrate how to circumvent this assumption by considering, as a case study,
        the problem of learning linear classifiers in $\re^d$.
        We show that in the case where the privacy status is correlated with the target label (as in the above example),
        % We show that in the case where the private and public data are labeled oppositely (as in the above example), 
        %We show that in the case where the private and public data are labeled differently (as in the above example), 
        linear classifiers in $\re^d$ can be learned, in the agnostic as well as the realizable setting, 
        with sample complexity which is comparable to that of the classical (non-private) PAC-learning. 
        It is known that this task is impossible if all the data is considered private.

\end{abstract}

\input{intro}

\input{prelim}
\input{def}
\input{realizable}
\ifnips
\input{broader}
\fi

\subsection*{Acknowledgements}
RB's research is supported by NSF Awards AF-1908281, SHF-1907715, and Google Faculty Research Award.

\bibliographystyle{alpha} 
\bibliography{reference}

\end{document}

%% file: intro.tex
\section{Introduction}\label{sec:intro}

%\anote{The entire data population can be viewed as 4 sub-populations partitioned on the target labels and privacy status}

Despite the remarkable progress in privacy-preserving machine learning powered by the rigorous framework of differential privacy (DP) \cite{DMNS06}, the current state of the art has several limitations. Most of the existing works on differentially private learning follow a conventional model, where the entirety of the input dataset to the learning algorithm is assumed to be sensitive and private, and hence, requires protection via the stringent constraint of DP. 
%focus on designing algorithms that satisfy the stringent constraint of differential privacy with respect to the whole input datasets. 
Unfortunately, this conservative approach has fundamental limitations that manifest in many problems. For example, learning even simple classes of functions (e.g., one-dimensional thresholds over $\re$) is provably impossible under that stringent model \cite{bun2015differentially,almm19} even though such classes are trivially learnable without privacy constraints.  More recent works \cite{beimel2013private, bassily2018model,ABM19,nandi2020privately,bassily2020private} have considered a more relaxed model, where the input dataset is made up of two parts: a private sample (as in the conventional model), and a ``public'' sample that entails no privacy constraints. In this model, the algorithm is required to satisfy DP only with respect to the private sample. Despite of the good news brought by these works showing the possibility of circumventing some of the aforementioned limitations by harnessing a limited amount of public data, \emph{all these works make the strong assumption that both the private and public samples come from the same population (i.e., they arise from the same distribution).} This can limit the practical value of these results in many real-life scenarios, where private and public data are naturally distinct. 
%\snote{Perhaps make italic the statement ``all of these works make the strong assumption\ldots''?}

Indeed, for a data record, the attribute of being sensitive can be strongly correlated with the value of that record (i.e., the realization of the feature-vector and the label). For example, imagine a scenario where a bank wants to predict the credit-worthiness of applicants for a credit card. To do this, a training sample is drawn from a population of individuals with good and bad credit scores. Suppose individuals with good credit score have no privacy concerns in sharing their data with the bank (and hence, their data can be viewed as ``public''), while those with bad credit score are concerned about what the study may reveal about them to third parties and understandably so, because they do not want such information to impact their chances in future opportunities. In this example, the population is a mixture of two very different groups: a sub-population with a good credit score (public sub-population), and a sub-population with bad credit score (private sub-population).

In this work, we introduce a new model for differentially private learning in which a learning algorithm has access to a mixed dataset of private and public examples that arise from possibly different distributions. The algorithm is required to satisfy DP only with respect to the private examples. More specifically, in our model, the underlying population (data distribution) $\cD$ is a \emph{mixture} of two possibly \emph{different} sub-populations: a private sub-population $\Dprv$ of sensitive data, and public sub-population $\Dpub$ of data that is deemed by its original owner to have no risk to personal privacy. We assume that each example drawn from the mixture $\cD$ has a ``privacy flag'' which is a binary label to indicate whether the example is private or public. As usual in the statistical learning framework, we do not assume the knowledge of $\cD$ or any of the sub-populations (or their respective weights in the mixture). 

\subsection*{Contributions}
\begin{itemize}[leftmargin=*]
\item \textbf{Introducing PPM model:} We formally describe the basic model of supervised learning from mixtures of private and public populations, and define the corresponding class of learning algorithms, which we refer to as \emph{Private-Public Mixture (PPM)} learners.  

\item \textbf{Learning Halfspaces:} Although the first quick impression about the model might be that it is a bit too general to allow for interesting results beyond what is covered by the conventional model of DP learning, we demonstrate that this is not the case and prove the first non-trivial result under this model in the context of learning halfspaces (linear classifiers) in $\re^d$ (for any $d\geq 1$). We give a construction of a PPM learner for this problem in the case where the privacy status is correlated with the target label, as in the credit-worthiness example above.
%, where being ``private'' corresponds to ``bad'' credit score and vice versa (or as in the example in the abstract, where being private corresponds to being unhealthy). 
Curiously, our PPM learner is improper: it outputs a hypothesis (classifier) that can be described by at most $d$ halfspaces. We hence derive upper bounds on the sample complexity of this problem in both the realizable and agnostic settings. In particular, we show that halfspaces in $\re^d$ can be learned in the aforementioned PPM model up to (excess) error $\alpha$ using $\approx \frac{d^2}{\alpha}$ total examples in the realizable setting, and using $\approx \frac{d^2}{\alpha^2}$ total examples in the agnostic setting. As noted earlier, in the conventional model, where all the examples drawn from $\cD$ are considered private, this class cannot be learned (even for $d=1$).  
%\snote{The above paragraph is a bit long, perhaps shorten it? (not critical)}
\end{itemize}

\paragraph{Techniques:} The idea of our construction for learning halfspaces goes as follows. First, we use the public examples to define a \emph{finite} family of halfspaces $\wtc_{\pub}$. Then, we employ a useful tool from convex geometry known as Helly's Theorem %\cite{helly1923mengen} to argue the existence of a collection of at most $d$ halfspaces in $\wtc_{\pub}$ whose intersection is disjoint
\cite{helly1923mengen} to argue the existence of a collection of at most $d$ halfspaces in $\wtc_{\pub}$ whose intersection is disjoint from the ERM halfspace (the halfspace with smallest empirical error with respect to the entire set of training examples). This implies that there is a hypothesis (described by the intersection of at most $d$ halfspaces from $\wtc_{\pub}$) whose empirical error is not larger than that of the ERM halfspace. Hence, we reduce our learning task to DP learning of a \emph{finite} class $\cG$ that contains all possible intersections of at most $d$ halfspaces from $\wtc_{\pub}$. We note that the latter task is feasible since $\cG$ is a finite class \cite{KLNRS08}. The description here is rather simplified. The actual construction and our analysis entail more intricate details (e.g., we need to carefully analyze the generalization error since the class $\cG$ itself depends on the public part of the training set). One clear aspect from the above description is that our construction is of an improper learner since, in general, the output hypothesis is given by the intersection of at most $d$ halfspaces. Devising a construction of a proper learner for this problem is an interesting open question. 

    \paragraph{Related work:} As mentioned earlier there have been some works that studied utilizing public data in differentially private data analysis. In particular, the notion of differentially private PAC learning assisted with public data was introduced by Beimel et al. in \cite{beimel2013private}, where it was called ``semi-private learning.'' They gave a construction of a learning algorithm in this setting, and derived upper bounds on the private and public sample complexities. The paper \cite{ABM19} revisited this problem and gave nearly optimal bounds on the private and public sample complexities in the agnostic PAC model. The work of \cite{bassily2020private} introduced the related, but distinct, problem of differentially private query release assisted by public data, and gave upper and lower bounds on private and public sample complexities. All these prior works assumed that the private and public examples arise from the same distribution (i.e., $\Dprv=\Dpub$), and their constructions particularly exploited this assumption. To the best of our knowledge, our work %is the first to introduce a formal model for learning from mixtures of private and public populations that does not entail this
is the first to consider a formal model for learning from mixtures of private and public populations that does not entail this assumption, and our construction for halfspaces demonstrates that this assumption can be circumvented. It is also worth pointing out that, unlike the aforementioned prior work, a PPM learner is only assumed to have access to examples from the mixture distribution $\cD$ rather than access to examples from \emph{each} of $\Dprv$ and $\Dpub$. Hence, unlike prior work, our construction does not require certain number of examples from each sub-population; it only requires a certain \emph{total} number of examples from the mixed population. 

%The idea of our construction for learning halfspaces consists of two pieces: (i) the convex hull of (a subset of) the public examples does not intersect with the ERM halfspace (the halfspace with smallest empirical error with respect to the entire set of training examples), and (ii) such convex 

%% file: prelim.tex
\section{Preliminaries}\label{sec:prelim}
In this section, we introduce some notation, state some basic concepts from learning theory, and describe some geometric properties we use
throughout the paper.

%\subsection*{Notations} 
% For $n\in\mathbb{N}$, we use $[n]$ to denote the set $\{1,\ldots,n\}$.
% We use standard asymptotic notation $O,\Omega,o,\omega$.
% A function $f:\mathbb{N}\to[0,1]$ is said to be negligible if $f(n) = o(n^{-d})$ for every $d\in\mathbb{N}$.
% The statement ``$f$ is negligible'' is denoted by $f=\negl(n)$. 

%We first define the probably approximately correct (PAC) model of Valiant~\cite{Valiant84}. 
\paragraph{Notation:} For $n\in\mathbb{N}$, we use $[n]$ to denote the set $\{1,\ldots,n\}$.
We use standard notation from the supervised learning literature (see, e.g.~\cite{shalev2014understanding}). Let $\cX$ denote an arbitrary domain (that represents the space of feature vectors). Let $\cY = \{0,1\}$. A function $h : \cX \rightarrow \cY$ is called a concept/hypothesis. %and it labels data points in $\cX$ by either $0$ or $1$.
%\snote{Perhaps set $\cY=\{\pm 1\}$ so it won't be confused with the privacy bit?} 
A family of concepts (hypotheses) $\cC \subseteq \cY^\cX$ is called a concept/hypothesis class. A learning algorithm, receives as input i.i.d. samples generated from some arbitrary distribution $\cD$ over $\cX \times \cY$, and outputs a hypothesis $h\in\cY^{\cX}$.

\paragraph{Expected error:} The expected/population error of a hypothesis $h:\cX\rightarrow \cY$ with respect to a distribution $\cD$ over $\cX \times \cY$ is defined by $\err(h; \cD)\triangleq \ex{(x, y)\sim\cD}{\ind\left(h(x)\neq y\right)}$.

% \noindent If  $\err(h; \cD') \leq  \err(c; \cD') + \alpha$ we say that $h$  $\alpha$-approximates the target concept $c$.

\noindent A distribution $\cD$ is called {\it realizable by $\cC$} if there exists $h^*\in \cC$ such that $\err(h^*; \cD)=0$.
In this case, the data distribution $\cD$ over $\cX\times\cY$ is completely described by a distribution $\cD_{\cX}$ over $\cX$ and a true labeling concept $h^*\in\cC$. %For realizable distributions, the expected error of a hypothesis $h$ will be denoted by $\err\left(h; ~\left(\cD_{\cX}, h^*\right)\right)\triangleq \ex{x\sim\cD_{\cX}}{\ind\left(h(x)\neq h^*(x)\right)}.$ 

% \begin{defn}[PAC Learning~\cite{Valiant84}]
% Algorithm $\cA$ is an $(\alpha,\beta)$-PAC learner for $\cC \subseteq \{0,1\}^\cX$ with input sample size $n$ if for any distribution $\cD$ over $\cX \times \cY$, given an input sample $S \sim \cD^n$, with probability at least $1-\beta$ (over the choice of $S,$ and any possible internal randomness in $\cA$), algorithm $\cA$ outputs a hypothesis $\hat{h} \in \{0, 1\}^{\cX}$, satisfying 
% $$\err\left(\hat{h};\cD\right)\leq \min\limits_{h\in\cC}\err\left(h; \cD\right)+\alpha.$$  
% \noindent If $\hat{h} \in \cC$ then $\cA$ is called a proper PAC learner; otherwise, it is called an improper PAC learner.
% \end{defn}
% \snote{Do we need this definition? It is both well-known and not used in what follows, right?}

\paragraph{Empirical error:} The empirical error of an hypothesis $h:\cX\rightarrow \{0, 1\}$ with respect to a labeled dataset $S=\left\{(x_1, y_1), \ldots, (x_n, y_n)\right\}$ will be denoted by $\herr\left(h; S\right)\triangleq \frac{1}{n}\sum_{i=1}^n \ind\left(h(x_i)\neq y_i\right).$\\
\noindent The problem of minimizing the empirical error on a dataset (i.e.\ outputting an hypothesis in the class with minimal error) is known as Empirical Risk Minimization (ERM). %We use $h_S^\erm$ to denote the hypothesis that minimizes the empirical error with respect to a dataset $S$, $h_S^\erm \triangleq \argmin\limits_{h \in \cH} \widehat{\err}(h;S)$.

%We also let $\cZ^* = \cup_{n=1}^{\infty}\mathcal{Z}^n$. 
%Each example in $\cZ$ can be either public or private and is denoted by $1$ or $0$ i.e. $\cP = \{0,1\}$.

%We use $\cD$ to denote a  distribution over $\cX \times \cY \times \cP$, $\cD_{\cX}$ to denote the marginal distribution over $\cX$, and $\cD_{\cX \times \cY|\cP}$ to denote the conditional distribution of $\cX \times \cY$ given $\cP$. W.l.o.g. we assume that private data points have the privacy status set as $0$ and public data points as $1$. Let $\Dprv \equiv \cD_{\cX \times \cY|0}$ and $\Dpub \equiv \cD_{\cX \times \cY|1}$ denote the conditional distributions over $\cX \times \cY$ when the privacy status is set as private and public respectively.  
%We use $S\sim \cD^n$ to denote a sample/dataset $S=\left\{(x_1, y_1,p_1), \ldots, (x_n, y_n,p_n)\right\}$ of $n$ i.i.d. draws from $\cD$.  
%\annote{distribution over $\cX \times \cY$ only or $\cX \times \cY \times \cP$ }

% \paragraph{Expected disagreement:} The expected disagreement between a pair of hypotheses $h_1$ and $h_2$ with respect to a distribution $\cD_{\cX}$ over $\cX$ is defined as $\dis\left(h_1, h_2; ~\cD_{\cX}\right)\triangleq \ex{x\sim\cD_{\cX}}{\ind\left(h_1(x)\neq h_2(x)\right)}.$
%  \paragraph{Empirical disagreement:} The
% empirical disagreement between a pair of hypotheses $h_1$ and $h_2$ w.r.t. an unlabeled dataset $T=\left\{x_1, \ldots, x_n\right\}$ is defined as $\hdis\left(h_1, h_2; ~T\right)=\frac{1}{n}\sum_{i=1}^n\ind\left(h_1(x_i)\neq h_2(x_i)\right).$

%We define the geometric objects we use in this paper.
We next define the geometric concepts we use in this paper.

\paragraph{Halfspaces and Hyperplanes:} For  $\bw = (w_0, w_1,\ldots, w_d)\in \re^{d+1}$, let $h_\bw$ denote the halfspace defined as $h_\bw \triangleq \{ x \in \re^d :\sum_{i=1}^{d} w_ix_i \geq w_0\}$. We will overload the notation and use $h$ to denote both the halfspace and the corresponding binary hypothesis defined as the indicator function of the halfspace. In particular, whenever we write $h(x)$, we would be referring to $h$ as the binary hypothesis associated with the halfspace $h,$ namely $h(x) \triangleq \ind\left(x \in h \right)$. 
\noindent A pair of halfspaces $h_{\bw}$ and $h_{-\bw}$ will be loosely referred to as ``opposite'' halfspaces. A pair of opposite halfspaces $h_{\bw}$ and $h_{-\bw}$ intersect in the hyperplane $hp_{\bw}\triangleq\{x \in \re^d :\sum_{i=1}^{d} w_ix_i = w_0\}$. 
\noindent A finite set $S\subset \re^d$ is said to \emph{support} a halfspace $h_{\bw}$ if $S$ is contained in the \emph{hyperplane} $hp_{\bw}$.

%\noindent For $\ell\in [d]$, we say that a halfspace $h$ is $\ell$-dimensional if it lies in an affine subspace of $\re^d$ of dimension $\ell$\footnote{An affine subspace of dimension $\ell$ is isomorphic to $\re^{\ell}$}. In such case, $h$ is obtained by the intersection of $1+2(d-\ell)$ halfspaces in $\re^d$. 

% \noindent For $\ell\in [d]$, we say that a halfspace is $\ell$-dimensional if it is obtained by the intersection of a halfspace in $\re^d$ and an $\ell$-dimensional subspace of $\re^d$. In particular, an $\ell$-dimensional halfspace can be described as the intersection of $1+2(d-\ell)$ halfspaces in $\re^d$. 
%Whenever we use $h$, it is viewed as a halfspace associated by a $d+1$ dimensional vector $\bw$. In particular, $h$ refers to the halfspace given by $ \{ x \in \re^d :\sum_{i=1}^{d} w_ix_i \geq w_0\}$. When we use $h(x)$, it is viewed as a binary funtion/hypothesis, namely $h(x) = \ind\left(x \in h \right) $.

% \paragraph{Carath\'eodory's Theorem}
% A fundamental result in convex geometry is the Carath\'eodory's Theorem. Let $Y$ be a set of point in $\re^d$. If there is a point $x \in \re^d$ such that the convex hull of $Y$, denoted by $\conv(Y)$, contains $x$, then there exist $y_1,\ldots, y_{d+1} \in Y$ such that $x \in \conv(\{y_1,\ldots, y_{d+1}\})$. \rnote{Why do we need that?!}

\paragraph{Affine subspace:} A  non-empty subset $\as \subseteq \re^d$ is an affine subspace, if there exists a $u \in \as$ such that $\as - u = \{x-u \mid x \in \as\}$ is a linear subspace of $\re^d$. Moreover, we say that $\as$ is $k$-dimensional affine subspace, $1\leq k\leq d$, if the corresponding linear subspace $\as-u$ is $k$-dimensional.

Since differential privacy is central to this work, we conclude this section by stating its definition.
\begin{defn}[Differential Privacy \cite{DKMMN06,DMNS06, DR14}]
Let $\epsilon,\delta >0$. A (randomized) algorithm $M:\{\cX \times \cY\}^n \rightarrow \cR$ is $(\eps,\delta)$-differentially private if for all pairs of datasets $S, S' \in \{\cX \times \cY\}$ that differ in exactly one entry, and every measurable $\cO \subseteq \cR$, we have: $$\Pr \left(M(S) \in \cO \right) \leq e^\eps \cdot \Pr \left(M(S') \in \cO \right) +\delta.$$
%When $\delta=0$, it is known as \emph{pure} differential privacy, and parameterized only by $\eps$ in this case. 
\end{defn}

%% file: def.tex
\section{Model and Definitions}\label{sec:def}

%\paragraph{Data Distribution, Training Examples, and Privacy Constraints}

\vspace{0.2cm}

\noindent In this paper, we consider a model of privacy-preserving learning, where the input dataset is a mixture of private and public examples. We call such model \emph{Private-Public Mixture (PPM)} learning. We view each example in the input dataset as a triplet comprised of a feature vector $x\in\cX$, a target label $y\in\cY$, and a privacy status bit $p\in\cP\triangleq\{\prv, \pub\}$. The privacy status is a bit that describes whether the example is private ($p=\prv$) and hence requires protection via differential privacy, or public ($p=\pub$) and hence does not entail any privacy concerns. In this paper, the privacy status is used only to distinguish between the private and public portions of the dataset. We stress that the goal is to learn how to classify the target label (and not the privacy bit).
%\snote{Perhaps make the privacy an element of $\{T,F\}$, or anything other sensible set disjoint from the label-set $\cY$?}
%The prediction goal only concerns the target label. 

%Let $\cX$ denote the space of feature vectors, $\cY$ denote the set of labels, and 
In our formulation, the training examples are i.i.d.\ from a distribution $\cD$ over $\cZ \triangleq \cX \times \cY \times \cP$. Hence, the distribution $\cD$ is a mixture of a public sub-population $\Dpub \triangleq \cD_{\cX\times\cY|\pub}$ and private sub-population $\Dprv \triangleq \cD_{\cX\times\cY|\prv}$, where $\cD_{\cX \times \cY|p}$ denotes the conditional distribution of the $(x, y) \in \cX\times \cY$ given a privacy-status bit $p\in \cP$. A sample $S \sim \cD^n$ is a mixture of private and public examples that can be distinguished using the privacy-status bit. Hence, we can partition the dataset $S$ into: a private dataset $\Sprv \in (\cX\times\cY)^{\nprv}$ and a public dataset $\Spub\in (\cX\times \cY)^{\npub}$, where $\nprv+\npub=n$. We note that $\Sprv\sim \Dprv^{\nprv}$ and $\Spub\sim \Dpub^{\npub}$.

%Here, w.l.o.g., we assume that private data points have the privacy status set as $\prv$ and public data points as $\pub$. 
% a domain $\cX$, a label set $\cY=\{0, 1\}$, and a privacy-status set  $\cP=\{\prv, \pub\}$; 
\paragraph{The PPM Learning Model:}
A PPM learning model is described by the following components: 
\mbox{(i) a distribution} $\cD$ over $\cX\times \cY\times \cP$; (ii) a dataset of $n$ i.i.d. examples from $\cD$; (iii) a loss function $\ell:\cY\times\cY \rightarrow \re_{+}$, which we fix to be the binary loss function, i.e., $\ell(\hat{y}, y)\triangleq \ind(\hat{y}\neq y), ~\hat{y}, y\in \cY$;  and (iv) a PPM learning algorithm,
% \snote{It is a bit strange to define the algorithm as part of the learning problem (it is rather a part of the solution).}
which we define below:

\begin{defn}[$(\epsilon, \delta, n)$-PPM Learning Algorithm]\label{defn:ppm-alg}
Let $\eps, \delta \in (0, 1)$, $n\in\mathbb{N}$. An $(\epsilon, \delta, n)$-PPM learning algorithm is a randomized map $\cA:\left(\cX\times\cY\times\cP\right)^n\rightarrow \cY^{\cX}$ that maps datasets of size $n$ (of private and public examples) to binary hypotheses such that for any $\npub\leq n$ and any realization of the public portion of the input dataset $\Spub\in \left(\cX\times\cY\right)^{\npub}$, the induced algorithm $\cA(\cdot, \Spub)$ is $(\epsilon, \delta)$-differentially private (w.r.t. the private portion of the input dataset). 
\end{defn}

\paragraph{Expected error of a PMM algorithm $\cA$:} Let $\tcD$ be a distribution over $\cX\times\cY$. Let $\hat{h}$ denote the hypothesis produced by $\cA$ on input sample $S$ of size $n$. The expected error of a PPM algorithm w.r.t.\ $\tcD$ is defined as $\err(\cA(S); \tcD)=\ex{(x, y)\sim\tcD}{\ind(\hat{h}(x)\neq y)}$. Note that the distribution here is only over $\cX\times\cY$ since, as mentioned earlier, the goal is to learn how to classify the target label and not the privacy status. Namely, $\tcD$ is the distribution that is obtained from the original distribution $\cD$ by marginalizing over the privacy status bit $p$.

Generally speaking, the goal in PPM learning is to design a PPM algorithm whose expected error is as small as possible (with high probability over the input i.i.d.\ sample and the algorithm's internal randomness). 

As stated, the above model does not specify how we quantify the learning goal over the choice of the distribution and the sample size. This is done to maintain flexibility in defining the learning paradigm based on the PPM model. Indeed, one can make different choices about such quantifiers and their order, which would result in different modes of learnability. One standard paradigm, which we will adopt in Section~\ref{sec:hspaces}, is to assume that the learning algorithm has access to a fixed hypothesis class $\cC\subseteq \cY^{\cX}$ and require that the algorithm attains small excess error, i.e., require that the expected error incurred by the algorithm is close to the smallest expected error attained by a hypothesis in $\cC$ (as in (agnostic) PAC learning). However, we still need to specify how we will quantify this desired goal over the distribution $\cD$. One possibility is to insist on uniform learnability; namely, require that we design a PPM algorithm that given a sufficiently large sample is guaranteed to have a small excess error (not exceeding a prespecified level) w.r.t. \emph{all} distributions $\cD$ over $\cX\times\cY\times \cP$. However, this route will lead us back to the conventional DP learning since the family of all  distributions $\cD$ clearly subsumes those distributions where all examples are private. We thus propose a meaningful alternative, where we fix a specific \emph{conditional distribution} $\cD_{\cP\,|\cX\times\cY}$ of the privacy status bit $p\in \cP$ given labeled example $(x, y)\in \cX\times \cY$ and quantify over \emph{all} distributions $\tcD$ over $\cX\times\cY$. 

\paragraph{Privacy model $\cD_{\cP\,|\cX\times\cY}$:} The conditional distribution $\cD_{\cP|\cX\times\cY}$ is given by 
\ifnips
$\left\{\pr{}{p=b ~| ~(x, y)}:~ b\in\cP, (x, y)\in\cX\times\cY\right\}$.
\else
$$\left\{\pr{}{p=b ~| ~(x, y)}:~ b\in\cP, (x, y)\in\cX\times\cY\right\}.$$
\fi
In other words, it can be seen as a mapping taking an example $(x,y)$ to the conditional distribution of its privacy bit, 
$\pr{}{p=\cdot \vert x,y}$.
We refer to $\cD_{\cP|\cX\times\cY}$ as \emph{the privacy model}. Such conditional distribution captures how likely a labeled example $(x, y)$ to be sensitive (from a privacy perspective). As discussed earlier, in many practical scenarios, the attribute of being sensitive can strongly depend on the realization of the data record. 

\paragraph{Label-\stabname privacy model:} A special case of the above definition is when the privacy status is perfectly correlated with the target label (as in the examples discussed in the introduction and the abstract). Namely, in this case, we have $p=\prv \iff y=1$ with probability $1$ (or, $p=\pub \iff y=1$ with probability $1$). We refer to this privacy model as \emph{label-\stabname}. 
%\snote{label-correlated may be confusing, maybe label-determined?}

Next, we formally define one possible class of PPM learners based on the discussion above. 

\begin{defn}[$(\alpha, \beta, \eps, \delta)$-PPM learner for a class $\cC$ w.r.t.\ a privacy model $\cD_{\cP\,|\cX\times\cY}$]\label{defn:pp-learner}
Let $\cC \subseteq \cY^{\cX}$ be a concept class, let $\cD_{\cP\,|\cX\times\cY}$ be a privacy model, and let $\alpha, \beta, \epsilon, \delta \in(0,1)$. A randomized algorithm $\cA$ is an $(\alpha, \beta, \eps, \delta)$-PPM learner for $\cC$ w.r.t.\ $\cD_{\cP\,|\cX\times\cY}$ with sample size $n$ if the following conditions hold: 

\begin{enumerate}

  \item $\cA$ is an $(\epsilon, \delta, n)$-PPM learning algorithm (see Definition \ref{defn:ppm-alg}). \label{cond:2}
  
    \item For every distribution $\tcD$ over $\cX\times\cY$,
    given a dataset $S\sim \cD^n$ where $\cD=\tcD\times\cD_{\cP\,|\cX\times\cY}$, $\cA$ outputs a hypothesis $\hat{h}$ such that, with probability at least $1-\beta$ (over $S\sim \cD^n$ and the internal randomness of~$\cA$), 
\[\err\left(\hat{h};~\tcD\right)\leq \min\limits_{h\in\cC}\err\left(h; ~\tcD\right)+\alpha.\]

    %For any $\npub \leq n$ and any $\Spub \in (\cX\times\cY\times \{\pub\})^{\npub},$ $\cA(\cdot, \Spub)$ is $(\eps, \delta)$-differentially private (i.e., $\cA$ is $(\eps, \delta)$-differentially private w.r.t. the private portion of the input dataset). \label{cond:2}
    %\snote{Perhaps just write here that $\cA$ satisfies Definition \ref{defn:ppm-alg}?}
\end{enumerate}
%\snote{I don't see how this definition depends on the privacy model. It is not clear over which distributions on $\cX\times\cY\times\cP$ the quantification applies.}
When the first condition is satisfied with $\delta=0$ (i.e., pure differential privacy), we refer to $\cA$ as $(\alpha, \beta, \eps)$-PPM learner for $\cC$ w.r.t.\ $\cD_{\cP\,|\cX\times\cY}$.
\end{defn}

In the special case of \emph{label-\stabname} privacy model, we say that $\cA$ is an $(\alpha, \beta, \eps, \delta)$-PPM learner for a class $\cC$ assuming label-\stabname privacy model. 

% As a special case of the above definition, we say that an algorithm $\cA$ is an $(\alpha, \beta, \eps, \delta)$-PPM learner for a class $\cC$ under the realizability assumption if it satisfies the first condition with respect to all distributions $\cD$ that are realizable by $\cC.$

% \begin{defn}[Public Data Assisted Privately Learnable Class]\label{defn:pp-class-learn}
% We say that a class $\cH$ is public data assisted privately learnable if there for every $\npub \in \mathbb{N}$ there exists a function $\nprv:(0, 1)^2\rightarrow \mathbb{N}$, and there is an algorithm $\cA$ such that for every $\alpha, \beta \in (0, 1),$ when $\cA$ is given private and public samples of sizes $\nprv=\nprv(\alpha, \beta),$ and $\npub$,
% it $\left(\alpha, \beta, \eps,\delta \right)$-public data assisted privately learns $\cH$. 

% \end{defn}

%\anote{Should there be different definitions for the case when privacy label is same as the output label }

%% file: realizable.tex
\section{Learning Halfspaces}\label{sec:hspaces}

We consider the problem of PPM learning for one of the most well-studied tasks in machine learning, namely, learning halfspaces (linear classifiers) in $\re^d$. We focus on the case of \emph{label-\stabname privacy model} defined earlier; that is, we consider the case where the privacy-status bit is perfectly correlated with the target label. In particular, a $1$-labeled data point is considered to be a private data point and a $0$-labeled data point is a public data point. We give a construction for a PPM learner for halfspaces in this case in both the agnostic and realizable settings. Our construction outputs a hypothesis with excess true error $\alpha$ using an input sample of size $\tilde{O}\left(\frac{d^2}{\eps\alpha}\right)$ in the realizable setting, and a sample of size $\tilde{O}\left(d^2\max\left(\frac{1}{\alpha^2},\frac{1}{\eps\alpha}\right)\right)$ in the agnostic setting. 
Our algorithm is an improper learner; %namely, 
specifically, the output hypothesis is given by the intersection of at most $d$ halfspaces. 

\paragraph{Relaxations to the label-\stabname privacy model:} Since perfect correlation between the privacy status and the target label might be a strong assumption to make in some practical scenarios, it is important for us to point out that such strict correlation is not necessary. In particular, our results (with exactly the same construction) still hold under either one of the following two relaxations to this assumption: (i) when the privacy status is only sufficiently correlated with the output label (in this case, the same construction will yield essentially the same accuracy since the impact of this relaxation on the excess error will be small); or (ii) when only the private examples have the same target label while the set of public examples can have both labels (in fact, our analysis will be exactly the same in this case). However, to emphasize the conceptual basis of our construction and maintain clarity and simplicity of the analysis, we opt to present the results for the simpler model that assumes perfect correlation. 

\ifnips
\paragraph{Overview:}
\else
\subsection*{Overview}
\fi
The input to our private algorithm is a dataset $S \in (\cX \times \cY \times \cP)^n$. %and the class $\cC$ of halfspaces in~$\re^d$. Here, we have $\cX=\re^d$. 
The dataset $S$ is partitioned into: $\Sprv \in (\cX \times \cY)^{\nprv}$ (private dataset) and $\Spub\in (\cX \times \cY)^{\npub}$ (public dataset) using the privacy-status bit in $\cP$, as described in Section~\ref{sec:def}, where $\nprv+\npub=n$. The main idea of our algorithm is to construct a family of halfspaces in~$\re^d$, denoted by $\wtc_{\pub}$, using the (unlabeled) public data points, and then restrict the algorithm to a finite hypothesis class $\cG$ made up of all intersections of at most $d$ halfspaces from $\wtc_{\pub}$. That is, each hypothesis in the finite hypothesis class $\cG$ is represented by an intersection of at most $d$ halfspaces from the family $\wtc_{\pub}$. Using Helly's Theorem \cite{helly1923mengen,radon1921mengen}, we can show that $\cG$ will contain one hypothesis whose error is comparable to that of the ERM halfspace. Hence, given the finite hypothesis class $\cG$, we construct a private learner that outputs a hypothesis from $\cG$ via the exponential mechamishm \cite{MT07}. Our construction is described formally in Algorithm~\ref{alg:real}.

\ifnips
\else
\vspace{0.2cm}
\fi

First, let's start by describing the construction of $\wtc_{\pub}$ and the finite hypothesis class $\cG$.

\ifnips
\else
\vspace{0.2cm}
\fi

Let $\uSpub\in\cX^{\npub}$ denote the \emph{unlabeled} version of the public portion $\Spub$ of the input dataset. The family of halfspaces $\wtc_{\pub}$ is constructed as follows. %Let $\npub$ denote the size of $\Spub$. 
\ifnips
Let $\cW \triangleq \{\hS\subseteq \uSpub: ~|\hS|\leq d \}.$
\else
Let 
$$\cW \triangleq \{\hS\subseteq \uSpub: ~|\hS|\leq d \}.$$ 
\fi
Namely, $\cW$ is a collection of all the subsets of $\uSpub$ of at most $d$ points. Note that the size of such collection is $|\cW|=O(\npub^d)$. For each $\hS\in \cW$, we find \emph{one} arbitrary halfspace in $\re^d$ that is \emph{supported} by $\hS$, and its corresponding opposite halfspace. We add these two halfspaces to $\wtc_{\pub}$. 
% Let $\as$ be the the affine subspace (as defined in Section~\ref{sec:prelim}) that contains all the data points in $\uSpub$. Lastly, together with all these halfspaces we will add this subspace $\as$ to $\wtc_{\pub}$. Note that, when the points of $\uSpub$ are in general position, we take $\as = \re^d$ and hence, this set is merely needed when the public data points lie in a lower dimensional affine space.
In addition to $\wtc_{\pub}$, we also define the affine subspace $\as$ that is spanned by the points in $\uSpub$ (where the notion of an affine subspace is as defined in Section~\ref{sec:prelim}). Note that, when the points of $\uSpub$ are in general position, $\as$ is trivially taken to be the entire $\re^d$. The set $\as$ is merely needed when the public data points lie in a lower dimensional affine subspace since in this case, we can simply restrict ourselves to the intersections of the halfspaces in $\wtc_{\pub}$ with $\as$.
Finally, we get a family of halfspaces $\wtc_{\pub}$  whose size is $|\wtc_{\pub}| =2\,|\cW| = O(\npub^d)$, and one additional set $\as$. We remark that if there are no public examples in the dataset, (i.e., $\uSpub=\emptyset$), then we simply return the empty set, i.e., $\wtc_{\pub}=\emptyset$. We formally describe the construction of $\wtc_{\pub}$ and $\as$ in Algorithm~\ref{alg:conC} (denoted by $\cA_\conC$).

% Namely, \cW_{\aff} is a collection of all affinely independent subsets of $\uSpub$ of at most $d$ points. Note that the size of such collection is $|\cW_{\aff}|=O(\npub^d)$. For each $\hS\in \cW_{\aff}$, we find the pair of opposite $|\hS|$-dimensional halfspaces whose intersection is the unique hyperplane containing $\hS$ (note that $|\hS|\leq d$; in particular, $|\hS|=d$ when $\hS$ is a set of $d$ points in general position); and add these two halfspaces to $\wtc_{\pub}$. Finally, we get a family $\wtc_{\pub}$ of halfspaces whose size is $|\wtc_{\pub}| =2\,|\cW_{\aff}|= O(2 \npub^d)$. We formally describe the construction of $\wtc_{\pub}$ in Algorithm~\ref{alg:conC} (denoted by $\cA_\conC$).

\begin{algorithm}[ht]
	\caption{$\cA_{\conC}$: Construction of the family $\wtc_{\pub}$ halfspaces}
	\begin{algorithmic}[1]
		\REQUIRE Dataset: $\Spub \in (\re^d \times \cY)^{\npub}$ %(with $\cX \subseteq \re^d$) 
		\STATE Let $\uSpub$ be the unlabeled version of $\Spub$.
		\STATE Initialize $\wtc_{\pub} = \emptyset$.
		%\STATE Let $\npub = |\Spub|$.
		\STATE Let $\cW = \{\hS\subseteq \uSpub: ~|\hS|\leq d\}.$
		\FOR{every $\hS \in \cW$:}
% 			\STATE Find the pair of opposite, $|\hS|$-dimensional halfspaces $h, h_{-}$ whose intersection is the unique hyperplane that contains $\hS$. \quad \COMMENT{The notion of opposite halfspaces is defined in Section~\ref{sec:prelim}.}
            \STATE Find a halfspace $h \in \re^d$ that is supported by $\hS$, and its corresponding opposite halfspace $h_{-}$.  \quad \COMMENT{The notion of opposite halfspaces is defined in Section~\ref{sec:prelim}.}\label{stp:pair-halfspaces}

			\STATE Add $h, h_{-}$ to $\wtc_{\pub}$.
		\ENDFOR
		\STATE Let $\as$ be the affine subspace spanned by $\uSpub.$ %~~ \COMMENT{So, $\as = \re^d$ when the affine space spanned by $\uSpub$ is .}
			%\STATE Add $\as$ to $\wtc_{\pub}$.
		\STATE Output $\{\wtc_{\pub},~\as\}$.
	\end{algorithmic}
	\label{alg:conC}
\end{algorithm}

\paragraph{Effective hypothesis class:}
In our main algorithm $\cA_\rel$ (Algorithm~\ref{alg:real} below), we construct a finite hypothesis class $\cG$ using $\wtc_{\pub}$ described above. Each hypothesis in $\cG$ corresponds to the intersection of at most $d$ halfspaces in the collection $\wtc_{\pub}$ and the affine subspace $\as$. Hence, it follows that $|\cG| \leq {|\wtc_{\pub}| \choose \leq ~d} = O(|\wtc_{\pub}|^d) = O(2^d~\npub^{d^2})$. 
Note that we consider the intersection of \emph{at most} $d$ halfspaces, so $\cG$ is assumed to also contain a hypothesis that corresponds to the empty set $\emptyset$, which assigns label $1$ to all points in $\re^d$ (according to our definition in Step~\ref{stp:hyp-G} of Algorithm~\ref{alg:real}).
% \new{Note again that in the case where there are no public examples in the dataset, $\cG$ will contain a single hypothesis that corresponds to the empty set $\emptyset$, which assigns label $1$ to all points in $\re^d$ (according to our definition in Step~\ref{stp:hyp-G} of Algorithm~\ref{alg:real}).}

	\begin{algorithm}[ht]
 		\caption{$\cA_\rel$: PPM Learning of Halfspaces }
 		\begin{algorithmic}[1]
 			\REQUIRE Class of halfspaces in $\re^d$: $\cC$; ~Labeled dataset: $S = \{(x_1, y_1, p_1), \ldots, (x_n, y_n, p_n)\}\in (\re^d \times \cY \times \cP)^n$, ~Privacy parameter: $\eps$
 			\STATE Initialize $\Spub \leftarrow \emptyset, ~S' \leftarrow \emptyset, ~\cG \leftarrow \emptyset$
 			\FOR{$i=1, \ldots, n$}\label{stp:const-Spub-1}
 			    \IF{$p_i = \pub$} \STATE Add $(x_i,y_i)$ to $\Spub$. \ENDIF
 			\ENDFOR\label{stp:const-Spub-2}
 			%\STATE Let $\npub = |\Spub|$.
 			%\IF{$\npub \leq d$}
 			%\STATE Set $\whh(x) = 1, \forall x \in \cX$.
 			%\STATE Output $\whh$
 			%\ELSE
 			\STATE $\{\wtc_{\pub}, ~\as\} \leftarrow \cA_\conC(\Spub)$. \label{stp:concC}
 			\FOR{$i=1, \ldots, n$}
 			\STATE Add $(x_i,y_i)$ to $S'$ \quad\COMMENT{$S'$ consists of all the $(x,y)$ pairs of $S$}\label{stp:constS'}
 			\ENDFOR
 			 		    
 			% \STATE For \new{every} sequence of halfspaces $h_1,\ldots, h_{\new{d}} \in \wtc_{\pub}$, add a hypothesis $g$ to $\cG$, where $g$ is defined as: $$g(x)\triangleq\ind\left(x\notin \left( \bigcap\limits_{i=1}^{\new{d}} h_i \cap \as \right) \right), ~~x\in \re^d.$$ \label{stp:hyp-G}
 			
 			\STATE For every $j\in [d]$, and every collection of distinct halfspaces $h_1,\ldots, h_j \in \wtc_{\pub}$, add a hypothesis $g$ to $\cG$, where $g$ is defined as: $$g(x)\triangleq\ind\left(x\notin \left( \bigcap\limits_{i=1}^{j} h_i \cap \as \right) \right), ~~x\in \re^d.$$  \label{stp:hyp-G}
 			
 			% \STATE For every $g \in \cG$ define the corresponding hypothesis as the indicator function of $\re^d\setminus g$:  $$\tlh(x) \triangleq \ind\left(x \notin \tlh\right), ~~x \in \cX. $$\label{stp:hyp-G}
 			\STATE Use the exponential mechanism with inputs $S',~\cG$, privacy parameter $\eps$, and a score function $q(S', g)\triangleq -\herr(g; S')$  to select a hypothesis $\whh$ from $\cG$.\label{stp:expMech}
 			\STATE Output $\whh$.
          %  \ENDIF
	
 		\end{algorithmic}
 		\label{alg:real}
 	\end{algorithm}

%\anote{Add case when $\npub < d$, output $0$ for all input}
\ifnips
\else
\noindent The privacy guarantee of $\cA_\rel$ is given by the following lemma. 
\fi
\begin{lem}[Privacy Guarantee of $\cA_\rel$]\label{lem:relHalfPriv}
For any realization of the privacy-status bits $(p_1, \ldots, p_n)$ $\in \cP^n$, and for any realization of $\Spub$ constructed in Steps~(\ref{stp:const-Spub-1} -\ref{stp:const-Spub-2}) of $\cA_\rel$ (Algorithm~\ref{alg:real}), $\cA_\rel$ is $\epsilon$-differentially private (w.r.t. the private portion of the input dataset). \end{lem}
\ifnips
The proof of the above lemma follows from the fact that $\{\wtc_{\pub},\as\}$ are constructed using only the public data together with the privacy analysis of the exponential mechanism \cite{MT07} (see details in the full version attached as supplementary material).
\else

\begin{proof}
	For any $\Spub \in (\cX \times \cY)^{\npub}$, the family of halfspaces $\wtc_{\pub}$ and the affine subspace $\as$ (Step \ref{stp:concC} in Algorithm~\ref{alg:real}) are constructed using only the public part of the dataset $S$ (and hence, so is $\cG$). The private part of $S$ is invoked in Step~\ref{stp:expMech}, which is an instantiation of the exponential mechanism. Thus, the proof follows directly from the privacy guarantee of the exponential mechanism \cite{MT07}. 
\end{proof}
\fi

Next, we turn to the analysis of the (excess) error of $\cA_{\rel}$. Let $\hrf$ denote the ERM halfspace with respect to the dataset $S'$; that is, $\hrf=\argmin\limits_{h\in \cC}\herr(h; S')$. We will first show that the expected error of the output hypothesis of $\cA_{\rel}$ is close to that of $\hrf$. Then, we derive explicit sample complexity bounds for $\cA_{\rel}$ in the realizable and agnostic settings. 

The first main step in our analysis is to show the existence of a hypothesis $g^{\ast}\in \cG$ whose empirical error is not larger than the empirical error of $\hrf$. 
% To do this, we first state some useful facts from convex geometry. 
Let $\uSpub\setminus \hrf\triangleq \{x\in \uSpub: x\notin \hrf\}$. First, we consider the corner case where $\uSpub\setminus\hrf=\emptyset$. In this case, all public examples are incorrectly labeled (i.e., assigned label $1$) by $\hrf$. Thus, the hypothesis $g^{\ast}\in \cG$ we are looking for is simply the empty hypothesis, which assigns label $1$ to all points in $\re^d$. Indeed, in such case the empirical error of $g^{\ast}$ cannot be larger than that of $\hrf$ since $g^{\ast}$ correctly labels all the private examples and is consistent with $\hrf$ on all the public examples.

Thus, in the remainder of our analysis, we will assume w.l.o.g. that $\uSpub\setminus\hrf\neq \emptyset$. We first state some useful facts from convex geometry.

For any finite set $T\subset \re^d$, we use $\cV(T)$ to denote the convex hull of all the data points in $T$. Note that
%assuming the points in $T$ are in general position, then 
$\cV(T)$ is a convex polytope that is given by the intersection of at most $O(|T|^d)$ halfspaces.
%\snote{There was an assumption here that $T$ is in general position which is not needed.}
% Let $\uSpub\setminus \hrf\triangleq \{x\in \uSpub: x\notin \hrf\}$. %(here, `$\setminus$' denotes the set difference operator). 
Hence, $\cV(\uSpub\setminus \hrf)$ is a convex polytope that contains all the public data points that are labeled correctly by $\hrf$. Moreover, $\cV(\uSpub\setminus \hrf)$ is given by the intersection of a sub-collection of halfspaces in $\wtc_{\pub}$ and $\as$.
%\snote{As we discussed, this is not entriely trivial (does not follow directly from the definitions). Perhaps we should rephrase and even find a reference (but it is not urgent for the deadline.)}
(As mentioned earlier, intersection with $\as$ is needed only when all the public data points lie in a lower dimensional affine subspace. In this case, the convex hull $\cV(\uSpub\setminus \hrf)$ is a ``flat'' set that lies in this affine subspace.) 
% \snote{I suggest to make the following remark explicit: for every finite $A\subseteq \re^d$, 
% each facet of the polytope $\cV(A)$ is the intersection of a $\mathsf{AffSpan(A)}$ with a definable halfspace of $\cV(A)$.
% (Where a halfspace of $\mathsf{AffSpan(A)})$ is definable if its supporting hyperplane is spanned by a subset of $A$.)}\rnote{This seems to be a bit convoluted. In your description you say ``intersection with a definable halfspace of $\cV(A)$'' but in your  definition you talk about a definable halfspace of $\mathsf{AffSpan}(A)$. Also, why do we need this elaboration about the facets?}
Thus, we can make the following immediate observation:
\ifnips
\else
\vspace{-0.3cm}
\fi
\begin{obv}\label{obv:empty}
Let $h_1, \ldots, h_v$ be halfspaces in $\wtc_{\pub}$ such that $\bigl(\bigcap\limits_{i=1}^{v} h_i\bigr) \cap \as=\cV(\uSpub\setminus \hrf)$. Then $ \Bigl(\bigl(\bigcap\limits_{i=1}^{v} h_i\bigr) \cap \as\Bigr) \cap \hrf= \emptyset$.
\end{obv}
A key step in our analysis relies on an application of a basic result in convex geometry known as Helly's Theorem, which we state below. 
\begin{lem}[Helly's Theorem restated \cite{helly1923mengen,radon1921mengen}]\label{lem:helly}
 Let $N \in \mathbb{N}$. Let $\cF = \{C_1,C_2,\ldots,C_N\}$ be a family of convex sets in $\re^d$. Suppose we have $\bigcap\limits_{i=1}^{N} C_i = \emptyset$, then there exists a collection $C_{i_1},\ldots,C_{i_{K}}$, where $K\leq d+1$, such that $C_{i_1} \cap \ldots \cap C_{i_{K}} = \emptyset$.
\end{lem}

% \begin{lem}[Helly's Theorem restated \cite{helly1923mengen,radon1921mengen}]\label{lem:helly}
%  For some $N \in \mathbb{N}$, let $h_1,h_2,\ldots,h_N$ be a finite collection of halfspaces in $\re^d$. If $\bigcap\limits_{i=1}^{N} h_i = \emptyset$, then there exists at most $d+1$ halfspaces $h_{i_1},h_{i_2},\ldots,h_{i_{d+1}}$ such that $h_{i_1} \cap \ldots \cap h_{i_{d+1}} = \emptyset$.
% \end{lem}

Combining Observation~\ref{obv:empty} and Lemma~\ref{lem:helly}, we obtain the following corollary:
\begin{cor}\label{cor:existT}
There exists a sub-collection of sets $\cT\subseteq \wtc_{\pub}\cup \{\as\}$, where $|\cT|\leq d$, such that $\left(\bigcap\limits_{h\in\cT} h\right) \cap \hrf  = \emptyset$.
\end{cor}
\begin{proof}
By Lemma~\ref{lem:helly} and Observation \ref{obv:empty}, there exists a sub-collection $\cT'\subseteq \{h_1, \ldots, h_v, \as,\hrf\}$
of size $\lvert \cT'\rvert \leq d+1$ such that the intersection of the sets in $\cT'$ is empty (where $h_1, \ldots, h_v$ are the halfspaces in Observation \ref{obv:empty}).
Observe that necessarily $\hrf\in \cT'$ since $\bigl(\bigcap\limits_{i=1}^v h_i\bigr)\cap \as=\cV(\uSpub \setminus\hrf)\neq\emptyset$. 
Therefore $\cT=\cT'\setminus\{\hrf\}$ gives the desired collection.
% \new{By Lemma~\ref{lem:helly} and Observation \ref{obv:empty}, there exists  a sub-collection $\cT'\subseteq \wtc_{\pub}\cup\{\as,\hrf\}$
% of size at at most $\lvert \cT'\rvert \leq d+1$ such that the intersection of the sets in $\cT'$ is empty.
% Observe that necessarily $\hrf\in \cT'$ (because $\bigl(\bigcap\limits_{h\in\cT} h\bigr)\cap \as\neq\emptyset$: e.g.\ it contains all $0$-labeled input examples), and therefore $\cT=\cT'\setminus\{\hrf\}$ gives the desired collection.}\rnote{this ``e.g. ...'' may be confusing. This does not necessarily contain all 0-labeled examples.}
\end{proof}
% \snote{The above proof implicitly assumes that there are $0$-labeled (i.e.\ public) examples.
% Does our algorithm work in the (trivial) case when all examples are $1$-labelled?}\rnote{Given the updates in this section, is it now fixed?}

Define $g^{\ast}(x)\triangleq \ind\bigl(x\notin \bigcap\limits_{h\in\cT} h\bigr),~ x\in\re^d$, where $\cT$ is the collection of at most $d$ sets whose existence is established in Corollary~\ref{cor:existT}. %Note that, when the public data points lie in a lower dimneional affine space we get $\as \in \cT$, and that $g^{\ast}\in\cG$. 
Note that $g^{\ast}\in\cG$. Given this definition of $g^{\ast}$, we note that all points in $S'$ that are labeled correctly by $\hrf$ are also labeled correctly by $g^{\ast}$. Indeed, for any private (i.e. $1$-labeled) data point $x$ that $\hrf$ labels correctly (i.e. $x \in \hrf$), we have $x \notin  \bigl(\bigcap\limits_{h\in\cT} h\bigr)$ by Corollary~\ref{cor:existT}. Hence, $g^{\ast}$ labels $x$ correctly. Conversely, for any public (i.e., 0-labeled) data point $x$ that $\hrf$ labels correctly (i.e. $x \notin \hrf$), we must have $x \in \cV(\uSpub\setminus\hrf)\subseteq \bigl(\bigcap\limits_{h\in\cT} h\bigr)$, where the last step follows from the definition of the collection $\cT$ in the proof of Corollary~\ref{cor:existT}. 
% \snote{How does this follow from Corollary~\ref{cor:existT}? There could be points that are neither in $\hrf$ nor in $\bigl(\bigcap\limits_{h\in\cT} h\bigr)$. I think we need to use here that $\bigl(\bigcap\limits_{h\in\cT}h\bigr)$ contains all $0$-labeled input examples (which follows from its definition rather than Corollary~\ref{cor:existT}).}\rnote{How about now?} 
Hence, $g^{\ast}$ also labels~$x$ correctly. This clearly implies that
\ifnips
$\herr(g^{\ast};S') \leq \herr(\hrf;S').$
\else
the empirical error of $g^{\ast}$ cannot exceed the empirical error of $\hrf$. We formally state this implication in the following lemma.

\begin{lem}\label{lem:Emphrf}
There exists a hypothesis $g^{\ast}\in\cG$ that satisfies
$$\herr(g^{\ast};S') \leq \herr(\hrf;S').$$
\end{lem}
\fi

\ifnips
Next, using the fact above together with the standard accuracy analysis of the exponential mechanism \cite{MT07,KLNRS08}, in the following claim we show that, with high probability, the empirical error of output hypothesis $\whh$ of $\cA_{\rel}$ is close to that of $\hrf$. The full details are deferred to the full version attached as supplementary material.

\begin{claim}[Excess Empirical Error of $\cA_{\rel}$]\label{clm:experr}
	Let $\alpha,\beta,\eps \in (0,1)$. Let $S' \in (\re^d \times \cY)^n$ be any realization of the dataset. For $n = O\left(\frac{d^2 \log(d/\eps\alpha) + \log(1/\beta)}{~\eps~\alpha}\right)$, with probability at least $1-\beta$ (over the randomness Step~\ref{stp:expMech} of $\cA_\rel$), $\cA_{\rel}$ outputs a hypothesis $\whh \in \cG$ that satisfies: 
$$\herr\left(\whh; S'\right) - \herr(\hrf;S')\leq  \alpha.$$	
\end{claim}

\else
Next, using the properties of the exponential mechanism, we can show that with high probability the empirical error of output hypothesis $\whh$ of $\cA_{\rel}$ is close to that of $g^{\ast}$.

\begin{lem}\label{lem:experr}
	Let $\alpha,\beta,\eps \in (0,1)$. Let $S' \in (\re^d \times \cY)^n$ be any realization of the dataset. For $n = O\left(\frac{d^2 \log(d/\eps\alpha) + \log(1/\beta)}{~\eps~\alpha}\right)$, with probability at least $1-\beta$ (over the randomness Step~\ref{stp:expMech} of $\cA_\rel$), $\cA_{\rel}$ outputs a hypothesis $\whh \in \cG$ that satisfies: 
$$\herr\left(\whh; S'\right) - \herr(g^{\ast};S')\leq  \alpha.$$	
\end{lem}

\begin{proof}

Note that $|\cG| = O(2^d~\npub^{d^2})$, and that the score function for the exponential mechanism is $-\herr(h; S')$, whose global sensitivity is $1/n$.

By standard accuracy guarantees of exponential mechanism \cite{MT07}, it follows that an input sample size 
$$ n = O\left(\frac{1}{\eps \alpha}\left(\log\left(|\cG|\right) + \log\left(\frac{1}{\beta}\right)\right)\right)$$
is sufficient to ensure that, w.p. $\geq 1-\beta$ (over the randomness Step~\ref{stp:expMech}), we have
$$\herr\left(\whh; S'\right) \leq \min\limits_{g\in\cG} \herr\left(g; S'\right) + \alpha,$$ 
which implies that $\herr\left(\whh; S'\right) \leq \herr\left(g^{\ast}; S'\right) + \alpha.$

Substituting the size of $\cG$, it follows that
\begin{align*}
n &= O\left(\frac{1}{\eps \alpha}\left(\log\left(|\cG|\right) + \log\left(\frac{1}{\beta}\right)\right)\right) = O\left(\frac{1}{\eps \alpha}\left(\log\left(2^d~\npub^{d^2}\right) + \log\left(\frac{1}{\beta}\right)\right)\right)\\
% &= O\left(\frac{1}{\eps \alpha}\left(d^2\log\left(2~\npub\right) + \log\left(\frac{1}{\beta}\right)\right)\right)\\
% &\leq O\left(\frac{1}{\eps \alpha}\left(d^2\log\left(n\right) + d^2\log\left(2\right) + \log\left(\frac{1}{\beta}\right)\right)\right)\\ %\intertext{as $\npub \leq n$}\\
&= O\left(\frac{1}{\eps \alpha}\left(d^2\log\left(\frac{d}{\eps \alpha}\right) +  \log\left(\frac{1}{\beta}\right)\right)\right).
\end{align*} 	
\end{proof}

By combining the two previous lemmas, we directly reach the following claim that asserts that the empirical error of the output hypothesis of $\cA_\rel$ is close to that of the ERM halfspace $\hrf \in \cC$. 
\begin{claim}[Excess Empirical Error of $\cA_{\rel}$]\label{clm:experr}
	Let $\alpha,\beta,\eps \in (0,1)$. Let $S' \in (\re^d \times \cY)^n$ be any realization of the dataset. For $n = O\left(\frac{d^2 \log(d/\eps\alpha) + \log(1/\beta)}{~\eps~\alpha}\right)$, with probability at least $1-\beta$ (over the randomness Step~\ref{stp:expMech} of $\cA_\rel$), $\cA_{\rel}$ outputs a hypothesis $\whh \in \cG$ that satisfies: 
$$\herr\left(\whh; S'\right) - \herr(\hrf;S')\leq  \alpha.$$	
\end{claim}
\fi Now the remaining ingredient in our analysis is to show that the generalization error of $\cA_\rel$ is also small. 
\ifnips 
We do this by observing that each hypothesis in $\cG$ can be described by a few data points from the input dataset and then invoking standard sample compression bounds \cite{littlestone1986relating, shalev2014understanding}. Specifically, we observe that each $g \in \cG$ is an intersection of at most $d$ halfspaces in $\wtc_{\pub}$  (restricted to $\as$), and each one of these halfspaces is represented by at most $d$ points from the input dataset. Hence, by putting all the ingredients together, we can finally arrive at our main result formally stated in the theorem below. Due to space considerations, we defer the details of the sample compression argument and the full proof of the theorem below to the supplementary material.
\begin{thm}[PPM learning of halfspaces]
	Let $\epsilon,\alpha,\beta \in (0,1)$. Assuming  label-\stabname privacy model, $\cA_\rel$ (Algorithm~\ref{alg:real}) is an $(\alpha,\beta,\epsilon)$-PPM learner for halfspaces in $\re^d$, with input sample size 
	$$n = O\left(\left(d^2 \log\left(\frac{d}{\eps\alpha}\right) + \log(\frac{1}{\beta})\right)\max\left(\frac{1}{\alpha^2},\frac{1}{~\eps~\alpha}\right)\right).$$ 
	Moreover, if we assume realizability, then $\cA_\rel$ is an $(\alpha,\beta,\epsilon)$-PPM learner for halfspaces in $\re^d$ with input sample size: 
	 $$n = O\left(\frac{d^2 \log(d/\eps\alpha) + \log(1/\beta)}{~\eps~\alpha}\right).$$ 
\end{thm}

\else
In fact, we will show that this is indeed the case for any algorithm that outputs a hypothesis in $\cG$. We observe that each hypothesis in $\cG$ is an intersection of at most $d$ halfspaces in $\wtc_{\pub}$ (possibly restricted to a lower dimensional affine subspace), and each one of these halfspaces is represented by at most $d$ points from the input dataset (Step~\ref{stp:pair-halfspaces} in Algorithm~\ref{alg:conC}). Hence, by using standard sample compression bounds \cite{littlestone1986relating, shalev2014understanding}, we can derive a bound on the generalization error of any algorithm that outputs any hypothesis in $\cG$. 

% By combining the two previous  lemmas, it is easy to see that the empirical error of the output hypothesis of $\cA_\rel$ is close to that of the ERM halfspace $\hrf \in \cC$. \fi Now the remaining ingredient in our analysis is to show that the generalization error of $\cA_\rel$ is also small. In fact, we will show that this is indeed the case for any algorithm that outputs a hypothesis in $\cG$. We observe that each hypothesis in $\cG$ is an intersection of at most $d$ halfspaces in $\wtc_{\pub}$ (possibly restricted to a lower dimensional affine subspace), and each one of these halfspaces is represented by at most $d$ points from the input dataset (Step~\ref{stp:pair-halfspaces} in Algorithm~\ref{alg:conC}). Hence, by using standard sample compression bounds \cite{littlestone1986relating, shalev2014understanding}, we can derive a bound on the generalization error of any algorithm that outputs any hypothesis in $\cG$. 

%Finally, by noting that $\cA_\rel$ outputs a hypothesis $\whh \in \cG$ this bound will also apply to $\whh$.
%\snote{I think we can improve the compression size from $d^2$ to $d$, but it is a bit involved (it builds on a technique from computational geometry called bottom-vertex-triangulation applied on a ``dual-space''; but perhaps it is too stressful to add it before the submission\ldots)}\rnote{Yes, it would be nice, but I agree that it can wait till after we submit.}

\begin{lem}[Sample Compression bound restated \cite{littlestone1986relating,shalev2014understanding}]\label{lem:sampCom}
	Let $k$ be an integer and let $\cB: (\cX \times \cY)^k \rightarrow \cG$ be a mapping from sequences of $k$ examples to the hypothesis class $\cG$. Let $\cA:  (\cX \times \cY)^n \rightarrow \cG$ be a learning rule that takes as input a dataset $S = ((x_1, y_1),\ldots,(x_n, y_n))$, and returns a hypothesis such that $\cA(S) = \cB((x_{i_1}, y_{i_1}),\ldots,(x_{i_k},y_{i_k}))$ for	some set of indices $(i_1,\ldots, i_k) \in [n]^k$.  Then for any distribution $\tcD$ over $\cX \times \cY$, with probability at least $1 - \beta$ (over $S \sim \tcD^n$), we have:
	$$\left|\err(\cA(S);\tcD)- \herr(\cA(S);S) \right| \leq \sqrt{\herr(\cA(S);S)\frac{4k\log(n/\beta)}{n}} + \frac{8k\log(n/\beta)}{n} + \frac{2k}{n}.$$
	
\end{lem}
Now we have all the ingredients to state and prove sample complexity bounds for our construction in both the realizable and agnostic settings. In the following statements, note that we already proved the privacy guarantee of $\cA_\rel$ in Lemma~\ref{lem:relHalfPriv}, and so we only focus on proving the sample complexity bounds.

\begin{thm}[PPM learning of halfspaces in the realizable case]\label{thm:AlgRel}
	
	 Let $\alpha,\beta,\epsilon \in (0,1)$. Assuming realizability, and assuming label-\stabname privacy model, $\cA_\rel$ (Algorithm~\ref{alg:real}) is an $(\alpha,\beta,\epsilon)$-PPM learner for halfspaces in $\re^d$ with input sample size: 
	 $$n = O\left(\frac{d^2 \log(d/\eps\alpha) + \log(1/\beta)}{~\eps~\alpha}\right).$$ 
\end{thm}

\begin{proof}
Let $\tcD$ be any distribution  over $\cX \times \cY$. Suppose $S' \sim \tcD^n$ (where $S'$ is a dataset in Step~\ref{stp:constS'} of Algorithm~\ref{alg:real}). %Note that combining Lemma~\ref{lem:Emphrf} and Lemma~\ref{lem:experr}, one can easily show that w.p. $\geq 1 - \beta/2$ (over randomness in Step~\ref{stp:expMech} of Algorithm~\ref{alg:real})
Note that by Claim~\ref{clm:experr}, we get that w.p. $\geq 1 - \beta/2$ (over randomness in Step~\ref{stp:expMech} of Algorithm~\ref{alg:real})
\begin{align*}
\herr\left(\whh; S'\right) -  \herr(\hrf;S') = \herr\left(\whh; S'\right) \leq  \frac{\alpha}{2}
%\herr\left(\whh; S'\right) \leq \herr(h^*;S') + \frac{\alpha}{2}\label{ineq:emp-errors}
\end{align*}
as long as $n = O\left(\frac{d^2 \log(d/\eps\alpha) + \log(1/\beta)}{~\eps~\alpha}\right),$ where here we used the fact that $\herr(\hrf;S')=0$ since this is the realizable setting.

Note that Lemma~\ref{lem:sampCom} (together with the argument before the statement of the lemma) immediately yields a bound on the generalization error of $\cA_\rel$ (with $k = d^2$ in the statement of the lemma). Namely, with probability $\geq 1-\beta/2$ (over the choice of $S' \sim \tcD^n$ and randomness in Step~\ref{stp:expMech} of Algorithm~\ref{alg:real}), we have: 
$$\lvert\err(\whh; \tcD)-\herr(\whh; S')\rvert \leq  \sqrt{\herr(\whh;S)\frac{4d^2\log(2n/\beta)}{n}} + \frac{8d^2\log(2n/\beta)}{n} + \frac{2d^2}{n}.$$
Now, using the bound on $\herr(\whh;S')$ above and for $n= O\left(\frac{d^2 \log(d/\eps\alpha) + \log(1/\beta)}{~\eps~\alpha}\right)$, we conclude that w.p. $\geq 1 - \beta,$ (over $S' \sim \tcD^n$ and the randomness in $\cA_\rel$), we have: $\err\left(\whh; \tcD\right) \leq \alpha.$

\end{proof}

\begin{thm}[PPM learning of halfspaces in the agnostic case]\label{thm:AlgAgn}
	Let $\epsilon,\alpha,\beta \in (0,1)$. Assuming  label-\stabname privacy model, $\cA_\rel$ (Algorithm~\ref{alg:real}) is an $(\alpha,\beta,\epsilon)$-PPM learner for halfspaces in $\re^d$, with input sample size 
	$$n = O\left(\left(d^2 \log\left(\frac{d}{\eps\alpha}\right) + \log(\frac{1}{\beta})\right)\max\left(\frac{1}{\alpha^2},\frac{1}{~\eps~\alpha}\right)\right).$$ 
\end{thm}

\begin{proof}

As in the proof of Theorem~\ref{thm:AlgRel}, with probability $\geq 1- \beta/4$ (over the randomness in Step~\ref{stp:expMech} of $\cA_\rel$), we have:
\begin{align}
 \herr\left(\whh; S'\right) - \herr(\hrf;S') \leq  \frac{\alpha}{4} \label{ineq:Agemp-err}
\end{align}
as long as $n= O\left(\frac{d^2 \log(d/\eps\alpha) + \log(1/\beta)}{~\eps~\alpha}\right).$
As before, Lemma~\ref{lem:sampCom} implies that with probability $\geq 1 - \beta/4$ (over the randomness in $S'$ and in $\cA_\rel$), we have
\begin{align*}
    \lvert\err(\whh; \tcD)-\herr(\whh; S')\rvert \leq  \frac{8d^2\log(4n/\beta)}{n} + \frac{2d^2}{n}.
\end{align*}
Hence, for $n= O\left(\frac{d^2 \log(d/\alpha) + \log(1/\beta)}{\alpha^2}\right)$, with probability $\geq 1- \beta/4$, we have:
\begin{align}
\lvert\err(\whh; \tcD)-\herr(\whh; S')\rvert&\leq \alpha/4,\label{ineq:gen-err-agn}
\end{align}

Moreover, by standard uniform convergence bounds \cite{shalev2014understanding}, for $n= O(\frac{d \log(1/\alpha) + \log(1/\beta)}{\alpha^2})$, with probability $\geq 1-\beta/2$ (over $S' \sim \tcD^n$), we have:
\begin{align}
\lvert \err\left(\hrf; \tcD\right)-  \herr\left(\hrf; S'\right)\rvert&\leq \alpha/4\label{ineq:gen-err-erm}\\
\err\left(\hrf; \tcD\right) - \min\limits_{h \in \cC}\err(h;\tcD) &\leq \frac{\alpha}{4} \label{ineq:ermh*}
\end{align}

Finally by combining (\ref{ineq:Agemp-err})-(\ref{ineq:ermh*}), and by the triangle inequality and the union bound, we conclude that for $n = O\left(\left(d^2 \log\left(\frac{d}{\eps\alpha}\right) + \log(\frac{1}{\beta})\right)\max\left(\frac{1}{\alpha^2},\frac{1}{~\eps~\alpha}\right)\right)$, with probability $\geq 1 - \beta$ (over the randomness in $S'$ and $\cA_\rel$), we have $\err\left(\whh; \tcD\right) - \min\limits_{h \in \cC}\err(h;\tcD) \leq \alpha.$ 

\end{proof}
\fi

%% file: broader.tex
\section*{Broader Impact}
Our work is theoretical in nature. Although there are no concrete, foreseeable ethical or societal impact for the research presented here, we hope that the new model we present for learning from mixtures of private and public populations could provide new insights that lead to a more realistic modeling for the problem of learning under privacy constraints, which, in turn, can lead to new practical privacy-preserving learning algorithms that meaningfully exploit data with no privacy concerns while providing strong privacy protection for sensitive data. Making progress in this direction can have significant impact on society in the long term.